\RequirePackage{fix-cm}
\documentclass[twocolumn]{svjour3}

\usepackage{amsfonts,amssymb,bm,mathtools}
\usepackage{multirow}
\usepackage[numbers]{natbib}

\def\makeheadbox{\relax}

\usepackage[pdftex]{graphicx}
\usepackage{times}
\usepackage{array}
\smartqed
\journalname{Neuroinformatics}

\begin{document}


\title{SegAN: Adversarial Network with Multi-scale $L_{1}$ Loss for Medical Image Segmentation\thanks{Yuan Xue and Tao Xu are Co-first Authors.}}


\author{Yuan Xue \and Tao Xu \and Han Zhang \and L. Rodney Long \and Xiaolei Huang }

\institute{Yuan Xue \and Tao Xu \and Xiaolei Huang \at
              Department of Computer Science and Engineering, Lehigh University, Bethlehem, PA, USA \\
              \email{\{yux715, tax313, xih206\}@lehigh.edu}           
           \and
           Han Zhang \at
              Department of Computer Science, Rutgers University, Piscataway, NJ, USA \\
              \email{han.zhang@cs.rutgers.edu}
           \and
            L. Rodney Long \at
              National Library of Medicine, National Institutes of Health, Bethesda, MD, USA \\
              \email{rlong@mail.nih.gov}
}

\date{Received: date / Accepted: date}

\maketitle
\begin{abstract}
Inspired by classic generative adversarial networks (GAN), we propose a novel end-to-end adversarial neural network, called SegAN, for the task of medical image segmentation.
Since image segmentation requires dense, pixel-level labeling, the single scalar real/fake output of a classic GAN's discriminator may be ineffective in producing stable and sufficient gradient feedback to the networks. Instead, we use a fully convolutional neural network as the segmentor to generate segmentation label maps, and propose a novel adversarial critic network with a multi-scale $L_{1}$ loss function to force the critic and segmentor to learn both global and local features that capture long- and short-range spatial relationships between pixels. In our SegAN framework, the segmentor and critic networks are trained in an alternating fashion in a min-max game:   The critic takes as input a pair of images, (original\_image $*$ predicted\_label\_map, original\_image $*$ ground\_truth\_label\_map), and then is trained by maximizing a multi-scale loss function; The segmentor is trained with only gradients passed along by the critic, with the aim to minimize the multi-scale loss function. We show that such a SegAN framework is more effective and stable for the segmentation task, and it leads to better performance than the state-of-the-art U-net segmentation method. We tested our SegAN method using datasets from the MICCAI BRATS brain tumor segmentation challenge. Extensive experimental results demonstrate the effectiveness of the proposed SegAN with multi-scale loss: on BRATS 2013 SegAN gives performance comparable to the state-of-the-art for whole tumor and tumor core segmentation while achieves better precision and sensitivity for Gd-enhance tumor core segmentation; on BRATS 2015 SegAN achieves better performance than the state-of-the-art in both dice score and precision.

\end{abstract}

%
\section{Introduction} \label{intro}
Advances in a wide range of medical imaging technologies have revolutionized how we view functional and pathological events in the body and define anatomical structures in which these events take place. X-ray, CAT, MRI, Ultrasound, nuclear medicine, among other medical imaging technologies, enable 2D or tomographic 3D images to capture in-vivo structural and functional information inside the body for diagnosis, prognosis, treatment planning and other purposes.


One fundamental problem in medical image analysis is image segmentation, which identifies the boundaries of objects such as organs or abnormal regions (e.g. tumors) in images. Since manually annotation can be very time-consuming and subjective, an accurate and reliable automatic segmentation method is valuable for both clinical and research purpose. Having the segmentation result makes it possible for shape analysis, detecting volume change, and making a precise radiation therapy treatment plan.

In
the literature of image processing and computer vision, various theoretical frameworks have been proposed for automatic segmentation. Traditional unsupervised methods such as thresholding~\cite{otsu1979threshold}, region growing~\cite{adams1994seeded}, edge detection and grouping~\cite{canny1986computational}, Markov Random Fields (MRFs)~\cite{manjunath1991unsupervised}, active contour models~\cite{kass1988snakes}, Mumford-Shah functional based frame partition~\cite{mumford1989optimal}, level sets~\cite{malladi1995shape}, graph cut~\cite{shi2000normalized}, mean shift~\cite{comaniciu2002mean}, and their extensions and integrations~\cite{gooya2011deformable, lee2008segmenting, lefohn2003interactive} usually utilize constraints about image intensity or object appearance.  Supervised methods~\cite{menze2015multimodal,cobzas20073d, geremia2011spatial, wels2008discriminative,ronneberger2015u, havaei2017brain}, on the other hand, directly learn from labeled training samples, extract features and context information in order to perform a dense pixel (or voxel)-wise classification.

Convolutional Neural Networks (CNNs) have been widely applied to visual recognition problems in recent years, and they are shown effective in learning a hierarchy of features at multiple scales from data. For pixel-wise semantic segmentation, CNNs have also achieved remarkable success.  In~\cite{long2015fully}, Long \emph{et al.} first proposed a fully convolutional networks (FCNs) for semantic segmentation.  The authors replaced conventional fully connected layers in CNNs with convolutional layers to obtain a coarse label map, and then upsampled the label map with deconvolutional layers to get per pixel classification results.  Noh \emph{et al.}~\cite{noh2015learning} used an encoder-decoder structure to get more fine details about segmented objects.
With multiple unpooling and deconvolutional layers in their architecture, they avoided the coarse-to-fine stage in~\cite{long2015fully}. However, they still needed to ensemble with FCNs in their method to capture local dependencies between labels.  Lin \emph{et al.}~\cite{lin2016efficient} combined Conditional Random Fields (CRFs) and CNNs to better explore spatial correlations between pixels, but they also needed to implement a dense CRF to refine their CNN output.

In the field of medical image segmentation, deep CNNs have also been applied with promising results.  Ronneberger \emph{et al.}~\cite{ronneberger2015u} presented a FCN, namely U-net, for segmenting neuronal structures in electron microscopic stacks.  With the idea of skip-connection from~\cite{long2015fully}, the U-net achieved very good performance and has since been applied to many different tasks such as image translation~\cite{isola2016image}.
In addition, Havaei \emph{et al.}~\cite{havaei2017brain} obtained good performance for medical image segmentation with their InputCascadeCNN. The InputCascadeCNN has image patches as inputs and uses a cascade of CNNs in which the output probabilities of a first-stage CNN are taken as additional inputs to a second-stage CNN.
Pereira \emph{et al.}~\cite{pereira2016brain} applied deep CNNs with small kernels for brain tumor segmentation.
They proposed different architectures for segmenting high grade and low grade tumors, respectively.
Kamnitsas \emph{et al.}~\cite{kamnitsas2017efficient} proposed a 3D CNN using two pathways with inputs of different resolutions. 3D CRFs were also needed to refine their results.

Although these previous approaches using CNNs for segmentation have achieved promising results, they still have limitations.  All above methods utilize a pixel-wise loss, such as softmax, in the last layer of their networks, which is insufficient to learn both local and global contextual relations between pixels.  Hence they always need models such as CRFs ~\cite{chen14semantic} as a refinement to enforce spatial contiguity in the output label maps. Many previous methods \cite{havaei2017brain,kamnitsas2017efficient,pereira2016brain} address this issue by training CNNs on image patches and using multi-scale, multi-path CNNs with different input resolutions or different CNN architectures. Using patches and multi-scale inputs could capture spatial context information to some extent.
Nevertheless, as described in U-net~\cite{ronneberger2015u}, the computational cost for patch training is very high and there is a trade-off between localization accuracy and the patch size.
Instead of training on small image patches, current state-of-the-art CNN architectures  such as U-net are trained on whole images or large image patches and use skip connections to combine hierarchical features for generating the label map. They have shown potential to implicitly learn some local dependencies between pixels. However, these methods are still limited by their pixel-wise loss function, which lacks the ability to enforce the learning of multi-scale spatial constraints directly in the end-to-end training process.
Compared with patch training, an issue for CNNs trained on entire images is label or class imbalance. While patch training methods can sample a balanced number of patches from each class, the numbers of pixels belonging to different classes in full-image training methods are usually imbalanced. To mitigate this problem, U-net uses a weighted cross-entropy loss to balance the class frequencies. However, the choice of weights in their loss function is task-specific and is hard to optimize. In contract to the weighted loss in U-net, a general loss that could avoid class imbalance as well as extra hyper-parameters would be more desirable.

In this paper, we propose a novel end-to-end Adversarial Network architecture, called SegAN, with a multi-scale $L_1$ loss function, for semantic segmentation. Inspired by the original GAN~\cite{goodfellow2014generative}, the training procedure for SegAN is similar to a two-player min-max game in which a segmentor network ($S$) and a critic network ($C$) are trained in an alternating fashion to respectively minimize and maximize an objective function. However, there are several major differences between our SegAN and the original GAN that make SegAN significantly better for the task of image segmentation.
\begin{itemize}
\vspace{-6pt}
\item
In contrast to classic GAN with separate losses for generator and discriminator, we propose a novel multi-scale loss function for both segmentor and critic.
Our critic is trained to maximize a novel multi-scale $L_{1}$ objective function that takes into account CNN feature differences between the predicted segmentation and the ground truth segmentation at multiple scales (i.e. at multiple layers).
\item We use a fully convolutional neural network (FCN) as the segmentor $S$, which is trained with only gradients flowing through the critic, and with the objective of minimizing the same loss function as for the critic. 
\item Our SegAN is an end-to-end architecture trained on whole images, with no requirements for patches, or inputs of multiple resolutions, or further smoothing of the predicted label maps using CRFs.
\end{itemize}

By training the entire system end-to-end with back propagation and alternating the optimization of $S$ and $C$, SegAN can directly learn spatial pixel dependencies at multiple scales. Compared with previous methods that learn hierarchical features with multi-scale multi-path CNNs~\cite{farabet2013learning}, our SegAN network applies a novel multi-scale loss to enforce the learning of hierarchical features in a more straightforward and efficient manner.
Extensive experimental results demonstrate that the proposed SegAN achieves comparable or better results than the state-of-the-art CNN-based architectures including U-net.

The rest of this paper is organized as follows. Section \ref{methods} introduces our SegAN architecture and methodology. Experimental results are presented in Section 3. Finally, we conclude this paper in Section 4.

\section{Methodology} \label{methods}

    \begin{figure*}[tb]
    \centerline{
    \includegraphics[width=0.99\textwidth,height = 350pt]{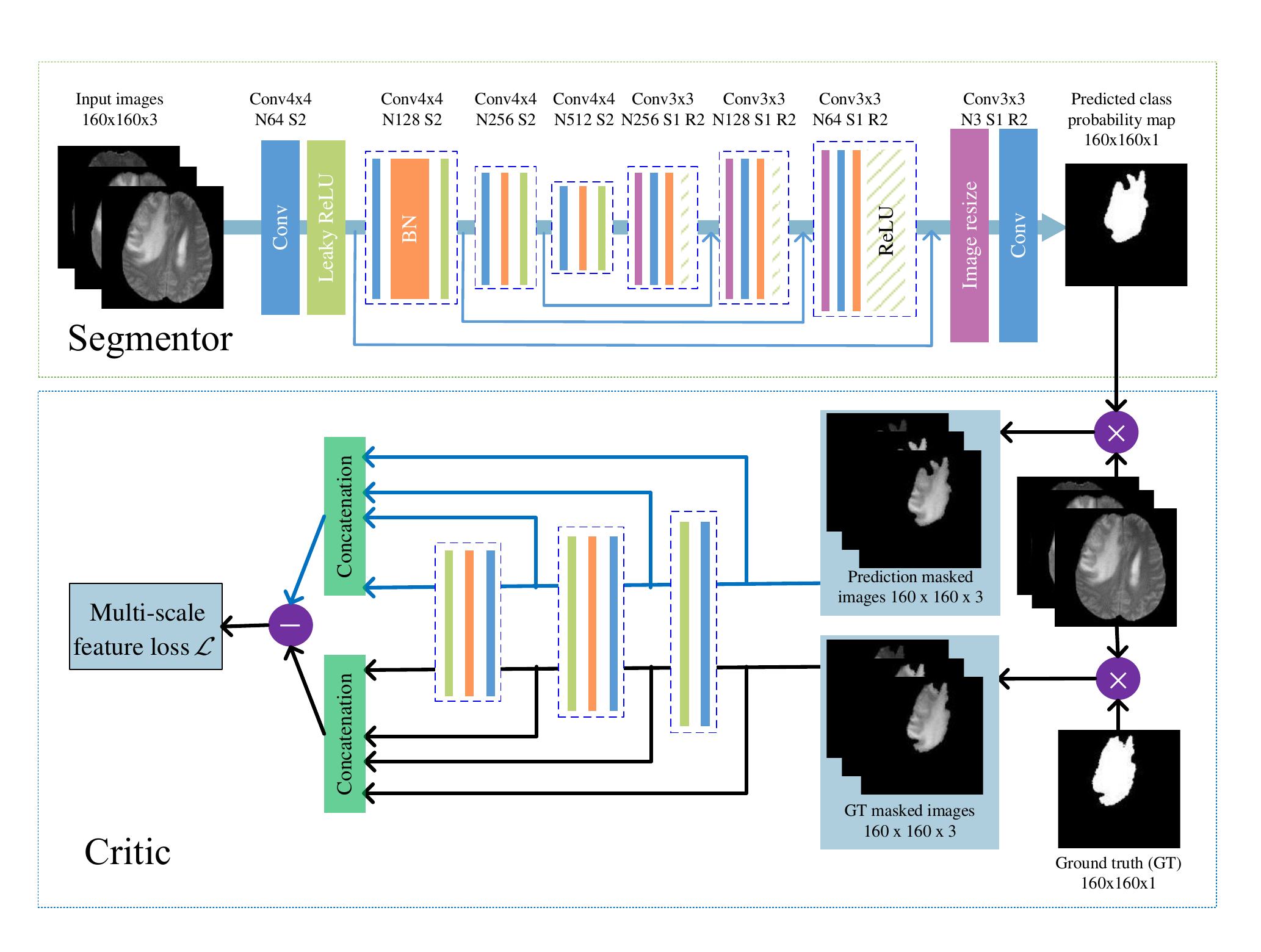}}
    \caption{The architecture of the proposed SegAN with segmentor and critic networks. $4 \times 4$ convolutional layers with stride $2$ (S2) and the corresponding number of feature maps (e.g., N64) are used for encoding, while image resize layers with a factor of $2$ (R2) and $3 \times 3$ convolutional layers with stride $1$ are used for decoding. Masked images are calculated by pixel-wise multiplication of a label map and (the multiple channels of) an input image. Note that, although only one label map (for whole tumor segmentation) is illustrated here, multiple label maps (e.g. also for tumor core and Gd-enhanced tumor core) can be generated by the segmentor in one path. }
    \label{fig:architecture}
		\end{figure*}

As illustrated in Figure~\ref{fig:architecture}, the proposed SegAN consists of two parts: the segmentor network $S$ and the critic network $C$.
The segmentor is a fully convolutional encoder-decoder network that generates a probability label map from input images.
The critic network is fed with two inputs: original images masked by ground truth label maps, and original images masked by predicted label maps from $S$.
The $S$ and $C$ networks are alternately trained in an adversarial fashion: the training of $S$ aims to minimize our proposed multi-scale $L_1$ loss, while the training of $C$ aims to maximize the same loss function.

\vspace{-8pt}
\subsection{The multi-scale $L_1$ loss}

The conventional GANs \cite{goodfellow2014generative} have an objective loss function defined as:
\begin{equation}
\begin{split}
\min_{\theta_G} &\max_{\theta_D} \mathcal{L}(\theta_G,\theta_D)\\
&= \mathbb{E}_{x\sim P_\mathrm{data}}[\log D(x)] + \mathbb{E}_{z\sim P_z}\log (1 - D(G(z)))] \enspace .
\end{split}
\end{equation}

In this objective function, $x$ is the real image from an unknown distribution $P_\mathrm{data}$, and $z$ is a random input for the generator, drawn from a probability distribution (such as Gaussion) $P_z$. $\theta_G$ and $\theta_D$ represent the parameters for the generator and discriminator in GAN, respectively.

In our proposed SegAN, given a dataset with $N$ training images $x_n$ and corresponding ground truth label maps $y_n$, the multi-scale objective loss function $\mathcal{L}$ is defined as:
\begin{equation}
\begin{split}
\min\limits_{\theta_S}  &\max\limits_{\theta_C}\mathcal{L}(\theta_S,\theta_C) \\
&= \frac{1}{N}\sum_{n=1}^N
\ell_\mathrm{mae}(f_C(x_n \circ S(x_n)), f_C(x_n \circ y_n)) \enspace ,
\label{eq:SegGANloss}
\end{split}
\end{equation}
where $\ell_\mathrm{mae}$ is the Mean Absolute Error (MAE) or $L_1$ distance;
$x_n \circ S(x_n)$ is the input image masked by segmentor-predicted label map (i.e., pixel-wise multiplication of predicted\_label\_map and original\_image); $x_n \circ y_n$ is the input image masked by its ground truth label map (i.e., pixel-wise multiplication of ground\_truth\_label\_map and original\_image); and $f_C(x)$ represents the hierarchical features extracted from image $x$ by the critic network. More specifically, the $\ell_\mathrm{mae}$ function is defined as:
\begin{equation}
\ell_\mathrm{mae}(f_C({x}), f_C({x}^{\prime})) = \frac{1}{L}\sum_{i=1}^L ||f_C^i({x}) - f_C^i({x}^{\prime})||_{1} \enspace ,\label{eq:mae}
\end{equation}
where $L$ is the total number of layers (i.e. scales) in the critic network, and $f_C^i({x})$ is the extracted feature map of image ${x}$ at the $i$th layer of $C$.

\subsection{SegAN Architecture}

\textbf{Segmentor.} We use a fully convolutional encoder-decoder structure for the segmentor $S$ network.
We use the convolutional layer with kernel size $4 \times 4$ and stride $2$ for downsampling, and perform upsampling by image resize layer with a factor of $2$ and convolutional layer with kernel size $3 \times 3$ stride $1$.
We also follow the U-net and add skip connections between corresponding layers in the encoder and the decoder.

\noindent\textbf{Critic.}
The critic $C$ has the similar structure as the decoder in $S$.
Hierarchical features are extracted from multiple layers of $C$ and used to compute the multi-scale $L_1$ loss.
This loss can capture long- and short-range spatial relations between pixels by using these hierarchical features, i.e., pixel-level features, low-level (e.g. superpixels) features, and middle-level (e.g. patches) features.

More details including activation layers (e.g.,  leaky ReLU), batch normalization layer and the number of feature maps used in each convolutional layers can be found in Figure~\ref{fig:architecture}.

\subsection{Training SegAN}

The segmentor $S$ and critic $C$ in SegAN are trained by back-propagation from the proposed multi-scale $L_1$ loss.  In an alternating fashion, we first fix $S$ and  train $C$ for one step using gradients computed from the loss function, and then fix $C$ and train $S$ for one step using gradients computed from the same loss function passed to $S$ from $C$.
As shown in (\ref{eq:SegGANloss}), the training of $S$ and $C$ is like playing a min-max game: while $G$ aims to minimize the multi-scale feature loss, $C$ tries to maximize it. As training progresses, both the $S$ and $C$ networks become more and more powerful. And eventually, the segmentor will be able to produce predicted label maps that are very close to the ground truth as labeled by human experts. We also find that the $S$-predicted label maps are smoother and contain less noise than manually-obtained ground truth label maps.

We trained all networks using RMSProp solver  with  batch  size $64$ and learning rate $0.00002$. We used a grid search method to select the best values for the number of up-sampling blocks and the number of down-sampling blocks for the segmentor (four, in both cases), and for the number of down-sampling blocks for the critic (three).

\vspace{-25pt}
\subsection{Proof of training stability and convergence}
\vspace{-10pt}
Having introduced the multi-scale $L_1$ loss, we next prove that our training is stable and finally reaches an equilibrium. First, we introduce some notations.

Let $f:\mathcal{X} \rightarrow \mathcal{X}^{\prime}$ be the mapping between an input medical image and its corresponding ground truth segmentation, where $\mathcal{X}$ represents the compact space of medical images\footnote[1]{Although the pixel value ranges of medical images can vary, one can always normalize them to a certain value range such as [0,1], so it is compact.} and $\mathcal{X}^{\prime}$ represents the compact space of ground truth segmentations. We approximate this ground truth mapping $f$ with a segmentor neural network $g_\theta : \mathcal{X} \rightarrow \mathcal{X}^{\prime}$ parameterized by vector $\theta$ which takes an input image, and generates a segmentation result. Assume the best approximation to the ground truth mapping by a neural network is the network $g_{\hat{\theta}}$ with optimal parameter vector $\hat{\theta}$.

Second, we introduce a lemma about the Lipschitz continuity of either the segmentor or the critic neural network in our framework.
\begin{lemma}\label{lem1}
Let $g_\theta$ be a neural network parameterized by $\theta$, and $x$ be some input in space $\mathcal{X}$, then $g_\theta$ is Lipschitz continuous with a bounded Lipschitz constants $K(\theta)$ such that
\begin{equation}||g_{\theta}(x_1) - g_{\theta}(x_2)||_1 \leqslant K(\theta)(||x_1 - x_2||_1) \enspace ,
\label{eq:lemma1_1}
\end{equation}
and for different parameters with same input we have
\begin{equation}
||g_{\theta_1}(x) - g_{\theta_2}(x)||_1 \leqslant K(x)||\theta_1 - \theta_2||_1
\enspace ,
\label{eq:lemma1_2}
\end{equation}
\end{lemma}

Now we prove Lemma~\ref{lem1}.
\begin{proof}
Note that the neural network consists of several affine transformations and pointwise nonlinear activation functions such as leaky ReLU (see Figure~\ref{fig:architecture}). All these functions are Lipschitz continuous because all their gradient magnitudes are within certain ranges. To prove Lemma~\ref{lem1}, it's equivalent to prove the gradient magnitudes of $g_\theta$ with respect to $x$ and $\theta$ are bounded. We start with a neural network with only one layer: $g_\theta(x) = A_1(W_1x)$ where $A_1$ and $W_1$ represent the activation and weight matrix in the first layer. We have $\nabla_{x}g_{\theta}(x) = W_1 D_1$ where $D_1$ is the diagonal Jacobian of the activation, and we have
$\nabla_{\theta} g_{\theta}(x) = D_1 x$ where $\theta$ represents the parameters in the first layer.

Then we consider the neural network with $L$ layers.
We apply the chain rule of the gradient and we have $\nabla_{x}g_{\theta}(x) = \prod_{k=1}^{L}W_k D_k$ where $k$ represent the $k$-th layer of the network. Then we have
\begin{equation}
\begin{split}
||\nabla_{x} g_{\theta}(x)||_1 = ||\prod_{k=1}^{L}W_k D_k||_1\enspace .
\end{split}
\end{equation}
Due to the fact that all parameters and inputs are bounded, we have proved (\ref{eq:lemma1_1}).

Let's denote the first $i$ layers of the neural network by $g^i$ (which is another neural network with less layers), we can compute the gradient with respect to the parameters in $i$-th layer as $\nabla_{\theta_i} g_{\theta}(x) = (\prod_{k=i+1}^{L}W_k D_k)D_ig^{i-1}(x)$. Then we sum parameters in all layers and get
\begin{equation}
\begin{split}
||\nabla_{\theta} g_{\theta}(x)||_1 &= ||\sum_{i=1}^{L}(\prod_{k=i+1}^{L}W_k D_k)D_ig^{i-1}(x)||_1\\
&\leqslant \sum_{i=1}^{L}||((\prod_{k=i+1}^{L}W_k D_k)D_i)g^{i-1}(x)||_1\enspace .
\end{split}
\end{equation}
Since we have proved that $g(x)$ is bounded, we finish the proof of (\ref{eq:lemma1_2}).

\end{proof}

Based on Lemma~\ref{lem1}, we then prove that our multi-scale loss is bounded and won't become arbitrarily large during the training, and it will finally converge.

\begin{theorem}
Let $\mathcal{L}_t(x)$ denote the multi-scale loss of our SegAN at training time $t$ for input image $x$, then there exists a small constant $C$ so that
\begin{equation}
\lim_{t\rightarrow + \infty} \mathbb{E}_{x \in \mathcal{X}} \mathcal{L}_t(x) \leqslant C\enspace .
\end{equation}
\label{thm:thm1}
\end{theorem}

\begin{proof}
Let $g$ and $d$ represent the segmentor and critic neural network, $\theta$ and $w$ be the parameter vector for the segmentor and critic, respectively. Without loss of generality, we omit the masked input for the critic and rephrase (\ref{eq:SegGANloss}) and (\ref{eq:mae}) as
\begin{equation}
\min_{\theta} \max_{w} \mathcal{L}_t = \mathbb{E}_{x \in \mathcal{X}} \frac{1}{L}\sum_{i=1}^L ||d^i(g_{\theta}({x})) - d^i(g_{\hat{\theta}}({x}))||_{1}\enspace ,
\label{eq9}
\end{equation}
recall that $g_{\hat{\theta}}$ is the ground truth segmentor network and $d^i$ is the critic network with only first $i$ layers. Let's firstly focus on the critic. To make sure our multi-scale loss won't become arbitrarily large, inspired by~\cite{arjovsky2017wasserstein}, we clamp the weights of our critic network to some certain range (e.g., $[-0.01,0.01]$ for all dimensions of parameter) every time we update the weights through gradient descent. That is to say, we have a compact parameter space $\mathcal{W}$ such that all functions in the critic network are in a parameterized family of functions $\{d_w\}_{w \in \mathcal{W}}$. From Lemma~\ref{lem1}, we know that $||d_{w}(x_1) - d_{w}(x_2)||_1 \leqslant K(w)(||x_1 - x_2||_1)$. Due to the fact that $\mathcal{W}$ is compact, we can find a maximum value for $K(w)$, $K$, and we have
\begin{equation}
||d(x_1) - d(x_2)||_1 \leqslant K||x_1 - x_2||_1\enspace .
\end{equation}
Note that this constant $K$ only depends on the space $\mathcal{W}$ and is irrelevant to individual weights, so it is true for any parameter vector $w$ after we fix the vector space $\mathcal{W}$. Since Lemma~\ref{lem1} applies for the critic network with any number of layers, we have
\begin{equation}
\frac{1}{L}\sum_{i=1}^L||d^i (g_{\theta}(x)) - d^i (g_{\hat{\theta}}(x))||_1 \leqslant K||g_{\theta}(x) - g_{\hat{\theta}}(x)||_1\enspace .
\label{eq11}
\end{equation}

Now let's move to the segmentor. According to Lemma~\ref{lem1}, we have $||g_{\theta}(x) - g_{\hat{\theta}}(x)||_1 \leqslant K(x)||\theta - \hat{\theta}||_1$, then combined with (\ref{eq11}) we have
\begin{equation}
\frac{1}{L}\sum_{i=1}^L||d^i (g_{\theta}(x)) - d^i (g_{\hat{\theta}}(x))||_1 \leqslant K(x)K||\theta - \hat{\theta}||_1\enspace .
\label{eq12}
\end{equation}
We know $\mathcal{X}$ is compact, so there's a maximal value for $K(x)$ and it only depends on the difference between the ground truth parameter vector $\hat{\theta}$ and the parameter vector of the segmentor $\theta$. Since we don't update weights in the segmentor when we update weights in the critic, there's an upper bound for $\mathcal{L}_t$ when we update the critic network and it won't be arbitrarily large during the min-max game.

When we update the parameters in the segmentor, we want to decrease the loss. This makes sense because smaller loss indicates smaller difference between $\hat{\theta}$ and $\theta$. When $\theta \rightarrow \hat{\theta}$, $\mathcal{L}_t$ converges to zero because the upper bound of $\mathcal{L}$ becomes zero. However, we may not be able to find the global optimum for $\theta$. Now let us denote a reachable local optimum for $\theta$ in the segmentor by $\theta_0$, we will keep updating parameters in the segmentor through gradient descent and gradually approaches $\theta_0$. Based on (\ref{eq9}) and (\ref{eq12}), we denote the maximum of $K(x)$ by $K^\prime$ and have
\begin{equation}
\lim_{t\rightarrow + \infty} \mathcal{L}_t(x) \leqslant K K^{\prime}||\hat{\theta} - \theta_0||_1 = C\enspace .
\end{equation}
Since the constant $C$ does not depend on input $x$, we have proved Theorem~\ref{thm:thm1}.
\end{proof}

\section{Experiments}

In this section, we evaluated our system on the fully-annotated MICCAI BRATS datasets~\cite{menze2015multimodal}.  Specifically,
we trained and validated our models using the BRATS 2015 training dataset, which consists of $220$ high grade subjects and $54$ low grade subjects with four modalities: T1, T1c, T2 and Flair.  We randomly split the BRATS 2015 training data with the ratio $9 : 1$ into a training set and a validation set. We did such split for the high grade and low grade subjects separately, and then re-combined the resulting sets for training and validation.  Each subject in BRATS 2015 dataset is a 3D brain MRI volume with size $240 \times 240 \times 155$.  We center cropped each subject into a subvolume of $180 \times 180 \times 128$, to remove the border black regions while still keep the entire brain regions. We did our final evaluation and comparison on the BRATS 2015 test set using the BRATS online evaluation system, which has \emph{Dice}, \emph{Precision} and \emph{Sensitivity} as the evaluation metrics. The Dice score is is identical to the F-score which normalizes the number of true positives to the average size of the two segmented regions:
\begin{equation}
\mathrm{Dice} = \frac{2|P \cap T|}{|P|+|T|}
\end{equation}
where $P$ and $T$ represent the predicted region and the ground truth region, respectively.
Since the BRATS 2013 dataset is a subset of BRATS 2015, we also present our results on BRATS 2013 leaderboard set.

Although some work with 3D patching CNNs have been done for medical image segmentation, due to the limitation of our hardware memory and for the reason that brain images in BRATS dataset are inconsistent in third dimension, we built a 2D SegAN network to generate the label map for each axial slice of a 3D volume and then restack these 2D label maps to produce the 3D label map for brain tumor. Since each subject was center cropped to be a $180 \times 180 \times 128$ volume, it yields 128 axial slices each with the size $180 \times 180$. These axial slices were further randomly cropped to size $160 \times 160$ during training for the purpose of data augmentation. They were centered cropped to size $160 \times 160$ during validation and testing.

We used three modalities of these MRI images: T1c, T2, FLAIR. Corresponding slices of T1c, T2, FLAIR modalities are concatenated along the channel dimension and used as the multi-channel input to our SegAN model, as shown in  Figure~\ref{fig:architecture}.
The segmentor of SegAN outputs label maps with the same dimensions as the input images.
As required by the BRATS challenge~\cite{menze2015multimodal}, we did experiments with the objective to generate label maps for three types of tumor regions: \emph{whole tumor}, \emph{tumor core} and \emph{Gd-enhanced tumor core}.

\begin{figure}[!htbp]
  \centering
    \includegraphics[width=0.48\textwidth]{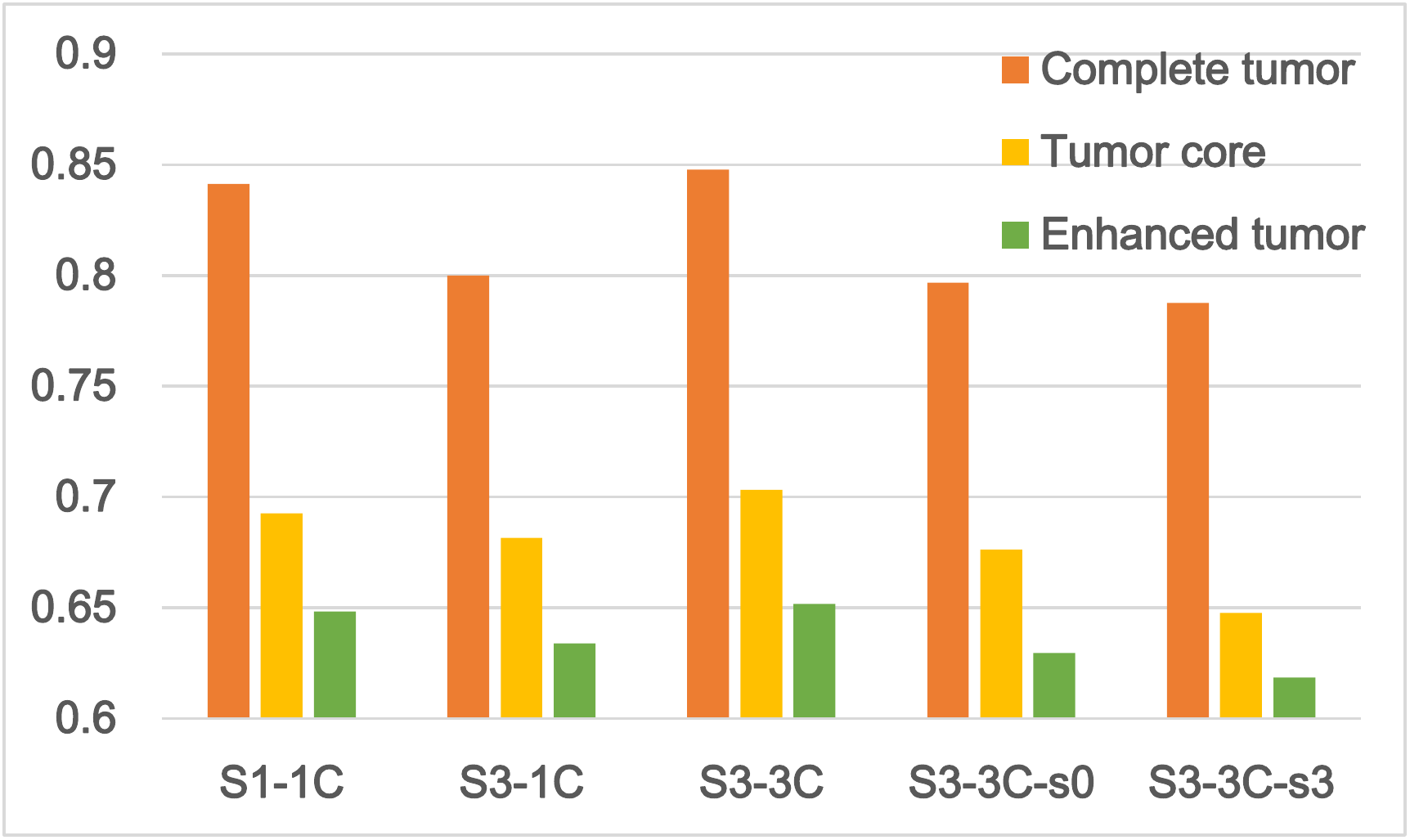}
    \caption{Average dice scores of different architectures on BRATS validation set}
		\label{fig:component}
\end{figure}
	
%

\subsection{Choice of components in SegAN architecture}

In this section, we compare different implementations of the proposed SegAN architecture and also evaluate the effectiveness of the proposed multi-scale $L_1$ loss on the BRATS validation set for the brain tumor segmentation task.
Specifically, we compare the following implementations:
\begin{itemize}
\vspace{-6pt}
\item \textbf{S1-1C.}  A separate SegAN is built for every label class, i.e. one segmentor (S1) and one critic (1C) per label.
\item \textbf{S3-1C:} A SegAN is built with one segmentor and one critic, but the segmentor generates a three-channel label map, one channel for each label. Therefore, each 3-channel label map produces three masked images (one for each class), which are then concatenated in the channel dimension and fed into the critic.
\item \textbf{S3-3C.} A SegAN is built with one segmentor that generates a three-channel (i.e. three-class) label map, but three separate critics, one for each label class. The networks, one $S$ and three $C$s, are then trained end-to-end using the average loss computed from all three $C$s.
\item \textbf{S3-3C single-scale loss models.} For comparison, we also built two single-scale loss models: S3-3C-s0 and S3-3C-s3. S3-3C-s0 computes the loss using features from only the input layers (i.e., layer $0$) of the critics, and S3-3C-s3 calculates the loss using features from only the output layers (i.e., layer $3$) of the critics.
\end{itemize}
As shown in Figure~\ref{fig:component}, models S1-1C and S3-3C give similar performance which is the best among all models. Since the computational cost for S1-1C is higher than S3-3C, S3-3C is more favorable and we use it to compare our SegAN model with other methods in Section \ref{sec:comparison}.
In contrast, while model S3-1C is the simplest requiring the least computational cost, it sacrifices some performance; but by using the multi-scale loss, it still performs better than any of the two single-scale loss models especially for segmenting tumor core and Gd-enhanced tumor core regions.

\begin{figure}
  \includegraphics[width=0.48\textwidth]{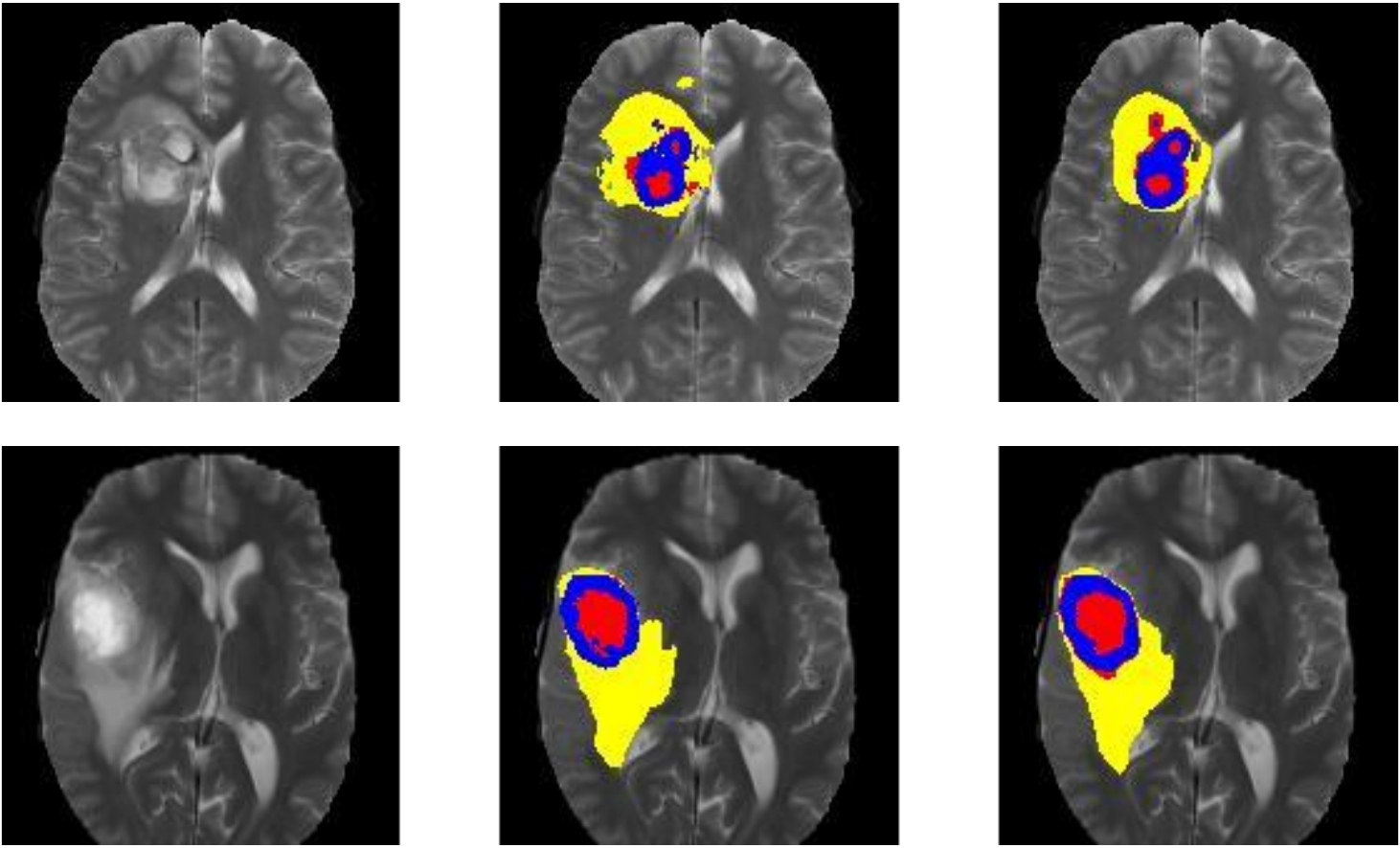}  
     \caption{Example results of our SegGAN (right) with corresponding T2 slices (left) and ground truth (middle) on BRATS validation set.}
 		\label{fig:examples}

 \label{sec:comparison}
\end{figure}

\begin{table*}[!htbp]
	\caption{Comparison to previous methods and a baseline implementation of U-net with softmax loss for segmenting three classes of brain tumor regions: whole, core and Gd-enhanced (Enha.)}
	\label{tb:result}
	\centering
	\begin{tabular}{c|c|c|c|c|c|c|c|c|c|c}	
		\cline{1-11}
\multirow{2}{2cm}{}&\multirow{2}{*}{Methods} & \multicolumn{3}{|c}{Dice} &
\multicolumn{3}{|c}{Precision} & \multicolumn{3}{|c}{Sensitivity} \\
		\cline{3-11}
		            & & Whole& Core & Enha. & Whole& Core & Enha. & Whole& Core & Enha.\\
		\hline					
\multirow{3}{1.8cm}{BRATS 2013 Leaderboard}
&Havaei~\cite{havaei2017brain} &
		   \textbf{0.84} &0.71 &0.57              &\textbf{0.88} &0.79 &0.54  &0.84 &0.72 &0.68 \\
		\cline{2-11}
&Pereira~\cite{pereira2016brain} &
			\textbf{0.84} &\textbf{0.72} &0.62     &0.85 &\textbf{0.82} &0.60  &\textbf{0.86} &\textbf{0.76} &0.68\\
		\cline{2-11}
&\textbf{SegAN} &
    	   \textbf{0.84} & 0.70 & \textbf{0.65}   &0.87 &0.80 &\textbf{0.68}  &0.83 & 0.74 &\textbf{0.72}\\
		\hline
		\hline

\multirow{3}{1.8cm}{BRATS 2015 Test} &Kamnitsas~\cite{kamnitsas2017efficient} &
            \textbf{0.85} &0.67 &{0.63}      &0.85 &\textbf{0.86} &0.63  &\textbf{0.88} &0.60 &\textbf{0.67}\\
        \cline{2-11}
&U-net &
         0.80 &0.63 &0.64             &0.83 &0.81 &\textbf{0.78}             &0.80 &0.58 &0.60\\
        \cline{2-11}
&\textbf{SegAN} &
    	   \textbf{0.85} &\textbf{0.70} &\textbf{0.66}  &\textbf{0.92} &0.80 &0.69 &0.80 &\textbf{0.65} &0.62\\

		\hline
	\end{tabular}
\end{table*}

\subsection{Comparison to state-of-the-art}
In this subsection, we compare the proposed method, our S3-3C SegAN model, with other state-of-the-art methods on the BRATS 2013 Leaderboard~\cite{havaei2017brain,pereira2016brain} Test and the BRATS 2015 Test~\cite{kamnitsas2017efficient}.
We also implemented a U-net model \cite{ronneberger2015u} for comparison. This U-net model has the exact same architecture as our SegAN segmentor except that the multi-scale SegAN loss is replaced with the softmax loss in the U-net. Table~\ref{tb:result} gives all comparison results. From the table, one can see that our SegAN compares favorably to the existing state-of-the-art on BRATS 2013 while achieves better performance on BRATS 2015.
Moreover, the dice scores of our SegAN outperform the U-net baseline for segmenting all three types of tumor regions.  Another observation is that our SegAN-produced label maps are smooth with little noise.  Figure~\ref{fig:examples} illustrates some example results of our SegAN; in the figure, the segmented regions of the three classes (whole tumor, tumor core, and Gd-enhanced tumor core) are shown in yellow, blue, and red, respectively. One possible reason behind this phenomenon is that the proposed multi-scale $L_1$ loss from our adversarial critic network encourages the segmentor to learn both global and local features that capture long- and short-range spatial relations between pixels, resulting fewer noises and smoother results.


\section{Discussion}

To the best of our knowledge, our proposed SegAN is the first GAN-inspired framework adapted specifically for the segmentation task that produces superior segmentation accuracy.
While conventional GANs have been successfully applied to many unsupervised learning tasks (e.g., image synthesis~\cite{zhang2016stackgan}) and semi-supervised classification~\cite{salimans2016improved}, there are very few works that apply adversarial learning to semantic segmentation. One such work that we found  by Luc \emph{et al.}~\cite{luc2016semantic} used both the conventional adversarial loss of GAN and pixel-wise softmax loss against ground truth.
They showed small but consistent gains on both the Stanford Background dataset and the PASCAL VOC 2012 dataset; the authors observed that pre-training only the adversarial network was unstable and suggested an alternating scheme for updating the segmenting network’s and the adversarial network’s weights.
We believe that the main reason contributing to the unstable training of their framework is: the conventional adversarial loss is based on a single scalar output by the discriminator that classifies a whole input image into real or fake category.
When inputs to the discriminator are generated \emph{vs.} ground truth dense pixel-wise label maps as in the segmentation task, the real/fake classification task is too easy for the discriminator and a trivial solution is found quickly. As a result, no sufficient gradients can flow through the discriminator to improve the training of generator.

In comparison, our SegAN uses a multi-scale feature loss that measures the difference between generated segmentation and ground truth segmentation at multiple layers in the critic, forcing both the segmentor and critic to learn hierarchical features that capture long- and short-range spatial relationships between pixels. Using the same loss function for both $S$ and $C$, the training of SegAN is end-to-end and stable.

\section{Conclusions}
%
In this paper, we propose a novel end-to-end Adversarial Network architecture, namely SegAN, with a new multi-scale loss for semantic segmentation. Experimental evaluation on the BRATS brain tumor segmentation dataset shows that the proposed multi-scale loss in an adversarial training framework is very effective and leads to more superior performance when compared with single-scale loss or the conventional pixel-wise softmax loss.

As a general framework, our SegAN is not limited to medical image segmentation applications. In our future work, we plan to investigate the potential of SegAN for general semantic segmentation tasks.

\section{Acknowledgements}
This research was supported in part by the Intramural Research Program of the National Institutes of Health (NIH), National Library of Medicine (NLM), and Lister Hill National Center for Biomedical Communications (LHNCBC), under Contract HHSN276201500692P.

\bibliographystyle{spbasic} \bibliography{references}

\end{document}